%% file: emnlp2018.tex
\documentclass[11pt,a4paper]{article}
\usepackage[hyperref]{emnlp2018}
\usepackage{times}
\usepackage{latexsym}

\usepackage{url}

\usepackage{graphicx}
\usepackage{float}
\restylefloat{table}
\usepackage{subcaption}
\usepackage{amsmath}
\usepackage{standalone}
\usepackage[colorinlistoftodos]{todonotes}
\usepackage{soul}
\usepackage{multirow}
\usepackage{rotating}
\usepackage{amsthm}
\usepackage{siunitx}
\usepackage{booktabs,siunitx,caption}

\usepackage{fancyhdr}
\usepackage{geometry}
\aclfinalcopy 

\newcommand{\m}[1]{\mathbf{#1}}

\newcommand{\ignore}[1]{}

\newtheorem{theorem}{Theorem}[section]

\newtheorem{lemma}[theorem]{Lemma}

\title{Adversarial Removal of Demographic Attributes from Text Data}

\author{Yanai Elazar$^\dagger$ \and Yoav Goldberg$^\dagger$$^\ast$ \\[0.5em]
$^\dagger$Computer Science Department,  Bar-Ilan University, Israel \\[0.5em]
$^\ast$Allen Institute for Artificial Intelligence \\[0.5em]
\texttt{\{yanaiela,yoav.goldberg\}@gmail.com} 
}

\date{}

\begin{document}
\maketitle
\begin{abstract}

  Recent advances in Representation Learning and Adversarial Training seem to succeed in removing unwanted features from the learned representation.
 We show that demographic information of authors is encoded in---and can be recovered from---the intermediate representations learned by text-based neural classifiers. 
The implication is that decisions of classifiers trained on textual
data are not agnostic to---and likely condition on---demographic attributes.
When attempting to remove such demographic information using adversarial training, we find that while the adversarial component achieves chance-level development-set accuracy during training, a post-hoc classifier, trained on the encoded sentences from the first part, still manages to reach substantially higher classification accuracies on the same data. This behavior is consistent across several tasks, demographic properties and datasets. We explore several techniques to improve the effectiveness of the adversarial component.
Our main conclusion is a cautionary one: do not rely on the adversarial training to achieve invariant representation to sensitive features.
\end{abstract}

\section{Introduction}
Consider automated systems that are used for determining credit ratings, setting insurance policy rates, or helping in hiring decisions about individuals. 
We would like such decisions to not take into account factors such as the gender or the race of the individual, or any other factor which we deem to be irrelevant to the decision.
We refer to such irrelevant factors as \emph{protected attributes}.
The naive solution of not including protected attributes in the features to a Machine Learning system is insufficient: other features may be highly correlated with---and thus predictive of---the protected attributes \cite{pedreshi2008discrimination}.
For example, in Credit Score modeling, text might help in credit score decisions \cite{ghailan2016improving}. By using the raw text as is, a discrimination issue might arise, as textual information can be predictive of some demographic factors \cite{hovy2015user} and author's attributes might correlate with target variables \cite{zhao2017men}.

In this paper we are interested in language-based features. It is well established that textual information can be predictive of age, race, gender, and many other social factors  of the author \cite{koppel2002automatically, burger2011discriminating, nguyen2013old,weren2014examining,verhoeven2014clips, rangel2016overview,verhoeven2016twisty,blodgett-green-oconnor:2016:EMNLP2016},  or even the audience of the text \cite{voigt2018rtgender}.

Thus, any system that incorporates raw text into its decision process is at risk of indirectly conditioning on such signals.
Recent advances in representation learning suggest \emph{adversarial
training} as a mean to hide the protected attributes from the decision function (Section \ref{sec:adversarial}). 
We perform a series of experiments and show that:
(1) Information about race, gender and age is indeed encoded into intermediate representations of neural networks, even when training for seemingly unrelated tasks and the training data is balanced in terms of the protected attributes (Section \ref{sec:baselines}); 
(2) The adversarial training method is indeed effective for reducing the amount of protected encoded information... (3) ...but in
some cases even though the adversarial component seems to be doing a perfect job, a fair amount of protected information still remains, and can be extracted from the encoded representations (Section \ref{sec:leakage-adv}).

This suggests that when working with text data it is very easy to condition on sensitive properties by mistake. 
Even when explicitly using the adversarial training method to remove such properties, one should not blindly trust the adversary, and be careful to ensure the protected attributes are indeed fully removed.
We explore means for improving the effectiveness of the adversarial training procedure
(section \ref{sec:improved-adv}).\footnote{The code and data acquisition are available in: \url{https://github.com/yanaiela/demog-text-removal}}

However, while successful to some extent, none of the methods fully succeed in removing all demographic information. Our main message, then, remains cautionary: \textbf{if the goal is to ensure fairness or invariant representation, do not trust adversarial removal of features from text inputs for achieving it.}

\section{Learning Setup}
\label{sec:setup}
We follow a setup in which we have some labeled data $D$ composed of documents
$x_1,...,x_n$ and task labels $y_1,...,y_n$.
We wish to train a classifier $f$ that accurately
predicts the main task labels $y_i$. Each data point $x_i$ is also associated with
a \emph{protected attribute} $z_i$, and we want the decision $y_i = f(x_i)$ to be oblivious to
$z_i$. 
Following \cite{ganin2015unsupervised,xie2017controllable}, we structure $f$ as an encoder $h(x)$ that maps $x$ into a representation
vector $\m{h}_x$, and a classifier $c(h(x))$ that is used for predicting $y$ based on
$\m{h}_x$. 
If $\m{h}_{x_i}$ is not predictive of $z_i$, then the main task prediction
$f(x_i)=c(h(x_i))$ does not depend on $z_i$. 

We say that a protected attribute $z$ has \emph{leaked} if we can train a classifier $c'(\m{h}_{x_i})$ to predict $z_i$ with an accuracy beyond chance level, and that the protected attribute is \emph{guarded} if we cannot train such a
classifier. We say that a classifier $f(x)=c(h(x))$ is \emph{guarded} if $z$ is guarded, and that it is \emph{leaky} with respect to $z$ if $z$ leaked.

\paragraph{Adversarial Training}
\label{sec:adversarial}
In order to make $f$ oblivious to $z$, we follow the adversarial training setup \cite{goodfellow2014generative, ganin2015unsupervised, beutel2017data, xie2017controllable}. During training, an adversarial classifier $adv(\m{h}_x)$
is trained to predict $z$, while the encoder $h$ is trained to make $adv$ fail.
Concretely, the training procedure tries to jointly optimize both quantities:
\begin{eqnarray*}
\arg\min_{adv}&L(adv(h(x_i)),z_i) \\
\arg\min_{h,c}&L(c(h(x_i)), y_i) - L(adv(h(x_i)),z_i)
\end{eqnarray*}
where $L(y',y)$ is the loss function (in our case, cross entropy).
This objective results in creating the representation $\m{h}_x$ s.t. it's maximally informative for the main task, while at the same time minimally informative of the protected attribute.
The optimization is performed in practice using the gradient-reversal layer
(GRL) method \cite{ganin2015unsupervised}. The GRL is a layer
$g_\lambda$ that is inserted between the encoded vector $\m{h}_x$ and the
adversarial classifier $adv$. During the forward pass the layer acts as the
identity, while during backpropagation it scales the gradients passed through it
by $-\lambda$, causing the encoder to receive the opposite gradients from the
adversary. The meta-parameter $\lambda$ controls the intensity of the reversal
layer. This results in the objective:
\[
\arg\min_{h,c,adv}L(c(h(x_i)), y_i) + L(adv(g_{\lambda}(h(x_i))), z_i)
\]
\\\noindent\textbf{Attacker Network} To test the effectiveness of the adversarial training,
we use an \emph{attacker network} $att(\m{h}_x)$. After the classifier $c(h(x))$
is fully trained, we use the encoder to obtain representations $\m{h}$, and
train the attacker network to predict $z$ based on $\m{h}$, without access to
the encoder or to the original inputs $x$ that resulted in $\m{h}$. 
If, after training, the attacker can predict $z$ on unseen examples with an accuracy of beyond chance level, then
the attribute $z$ leaked to the representation, and the classifier is \textit{not guarded}.\\
\noindent\textbf{Network Architecture} In our setup, an example $x_i$ is a sequence of tokens $w_1,...,w_{m_i}$ and the encoder is a one layer LSTM network that reads in the associated embedding vectors and returns the final state: $\m{h} = LSTM(\m{w}_{1:m})$.  The classifier $c$ and the adversarial $adv$ are both multi-layer perceptrons with one hidden layer, sharing the same hidden layer size and activation function (tanh).\footnote{Further details regarding the architecture and training parameters can be found in the supplementary materials.}

\section{Data, Tasks, and Protected Attributes}

To perform our experiments, we need a reasonably large dataset in which the data-points $x$ contain textual information, and for which we have both main-task labels $y$ and protected attribute labels $z$. 
While our motivating example used prediction tasks for credit rating, insurance rates or hiring decisions, to the best of our knowledge there are no publicly
available datasets for these sensitive tasks that meet our criteria. We thus
opted to use much less sensitive main-tasks, for which we can obtain the needed
data.
We focus on Twitter messages, and our protected attributes are binary-race
(non-hispanic Whites vs. non-hispanic Blacks), binary-gender (Male vs. Female)\footnote{While gender is a non-binary
construct, many decisions in the real-world are unfortunately still influenced
by hard binary gender categories. We thus consider binary-gender to be a useful
approximation in our context.} and binary-age (18-34 vs. 35+). As main tasks we
chose \emph{binary emoji-based sentiment prediction} and \emph{binary
tweet-mention prediction}.
Both the sentiment and the mention prediction tasks are not inherently correlated with race, gender or age.  Protected attributes leakage in these seemingly benign main-tasks is a strong indicator that such leakage is likely to occur also in more sensitive tasks.

\paragraph{Main Tasks: Sentiment and Mention-detection}
Both tasks can be derived automatically from twitter data.
We construct a binary ``sentiment'' task by identifying a subset of emojis which are associated with positive and negative sentiment,\footnote{Complete list is available in Appendix \ref{sec:emojis}} identifying tweets containing these emojis, assigning them with the corresponding sentiment and removing the emojis. Tweets containing emojis from both sentiment lists are discarded. The binary mention task is to determine if a tweet mentions another user, i.e, classifying conversational vs. non-conversational tweets. We derive this dataset by identifying tweets that include @mentions tokens, and removing all such tokens from the tweets.

\paragraph{Protected: Race}
The race annotation is based on the dialectal tweets (\textsc{Dial}) corpus from \cite{blodgett-green-oconnor:2016:EMNLP2016}, consisting of 59.2 million tweets by 2.8 million users.
Each tweet is associated with predicted ``race'' information which was predicted using a technique that takes into account the geo-location of the author and the words in the tweet. We focus on the AAE (African-American English) and SAE (Standard American English) categories, which we use as proxies for non-Hispanic blacks and non-Hispanic whites.

We chose only annotations with confidence (the probability of the authors' race) of above 80\%. Due to its construction, the race annotations in this dataset are highly correlated with the language being used. As such, the data reflects an extreme case in which the underlying language is very predictive of the protected attribute.

\paragraph{Protected: Age and Gender} We use data from the PAN16 dataset \cite{rangel2016overview}, containing manually annotated Age and Gender information of 436 Twitter users, along with up to 1k tweets for each user. User annotation was performed by consulting the user's LinkedIn profile.  Gender was determined by considering the user's name and photograph, discarding unclear cases. Age range was determined by birth-date which was published on the user's profile, or by mapping their degree starting date.

\paragraph{Data-splits}
From the \textsc{Dial} corpus we extracted 166K and 10K tweets for training and development purpose respectively (after cleaning and extracting relevant tweets), whereas for the PAN16 dataset we collected 160K tweets for training and 10K for development. 
The train/development split in both phases of the training (\textit{task-training} and \textit{attacker-training}) is the same.
This is the worst possible scenario for the attacker, as it is training on the exact representations the adversary attempted to remove the protected attribute from. Each split is balanced with respect to both the main and the protected labels: a random prediction of each variable is likely to result in 50\% accuracy.

\paragraph{Metrics}
Throughout this paper, we measure leakage using accuracy. We say that the protected attribute has leaked if an \textit{attacker} manages to predict the protected attribute with better than 50\% accuracy, which is always the probability of that attribute ($P(Z)=0.5$). In Appendix \ref{sec:fair-acc} we relate our metric to more standard fairness metrics, and prove that in our setup a guarded predictor guarantees demographic parity, equality of odds, and equality of opportunity.
Note however that we also show empirically that such guarded predictors are very hard to attain in practice.

\section{Baselines and Data Leakage}
\label{sec:baselines}
\paragraph{In-dataset Accuracy Upper-bounds}
We begin by examining how well can we perform on each task (both main-tasks and protected attributes) when training the encoder and classifier directly on that task, without any adversarial component. This provides an upper bound on the protected attribute leakage for the main tasks results. The results in Table \ref{tbl:direct-task} indicate that the classifiers achieve reasonable accuracies for the main tasks.\footnote{While the sentiment score may seem low, we manually verified the erroneous predictions and found out that many of them are indeed ambiguous with respect to sentiment, e.g. sentences like ``I can't take Amanda seriously \includegraphics[width=0.12in]{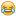}'' and ``You make me so angry, yet you make me so happy. \includegraphics[width=0.12in]{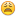}'' which were predicted negative and positive respectively, but their gold label was the opposite.} 
For the protected attributes, \emph{race} is highly predictable (83.9\%) while \emph{age} and \emph{gender} can also be recovered at above 64\% accuracy.
\input{tables/upper-bounds.tex}
\input{tables/leakage.tex}
\noindent\textbf{Leakage} When training directly for the protected
attributes, we can recover them with relatively high accuracies. But is information about them being encoded when we train on the main tasks? In this set of experiments, we encode the training and validation sets using the encoder trained on the main task, and train the attacker network to predict the protected attributes based on these vectors. This experiment suggests an upper bound on the amount of leakage of protected attributes when we do not actively attempt to prevent it.
The Balanced section in Table \ref{tbl:leakage} summarizes the validation-set accuracies.
While the numbers are lower than when training directly (Table \ref{tbl:direct-task}), they are still high enough to extract meaningful and possibly highly sensitive information (e.g. \textsc{Dial} Race direct prediction is 83.9\% while \textsc{Dial} Race leakage on the balanced Sentiment task is 64.5\%).

\paragraph{Leakage: Unbalanced Data}
\label{par:leakage}
\begin{figure}[t!]
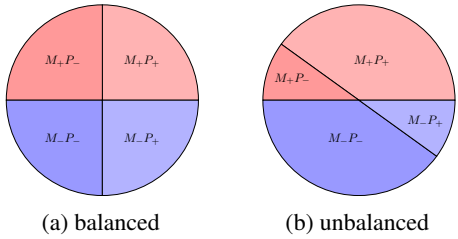

    \centering
    \begin{subfigure}[t]{0.2\textwidth}
        \centering
        \includegraphics[height=1.0in]{figures/dist.tex}
        \caption{balanced}
        \label{fig:balanced}
    \end{subfigure}%
    ~ 
    \begin{subfigure}[t]{0.2\textwidth}
        \centering
        \includegraphics[height=1.0in]{figures/dist-unbalanced.tex}
        \caption{unbalanced}
        \label{fig:unbalanced}
    \end{subfigure}
    \caption{Balanced (a) vs. Unbalanced (b) dataset. \textbf{Red(M+)/Blue(M-)}: Main Task. \textbf{Light(P+)/Dark(P-)}: Protected attribute. Each class is globally balanced, but in (b) the proportion of the protected attribute within each main task split is unbalanced.}
\end{figure} 

The datasets we considered were perfectly balanced
with respect to both main task and protected attribute labels (Figure \ref{fig:balanced}).
Such extreme case is not representative of real-world datasets, in which a dataset may be well balanced w.r.t. the main task labels but not the protected attribute. For example, when training a classifier to predict a fit for managerial position based on Curriculum Vitae (CV) of candidates, the CV dataset may be perfectly
balanced according to the managerial / non-managerial variable, but, because of existing social biases, CVs of females might be under-represented in the managerial category and over-represented in the non-managerial one.  In such a situation, the classifier may perpetuate the bias by learning to favor males over females for managerial positions. We simulate this more realistic scenario by constructing unbalanced datasets in which the main tasks (sentiment/mention) remain balanced but the protected class proportions within each main class are not, as demonstrated in Figure \ref{fig:unbalanced}. For example, in the sentiment/gender case, we set the positive-sentiment class to contain 80\% male and 20\% female tweets, while the negative-sentiment class contains
20\% male and 80\% female tweets.
We then follow the leakage experiment on the unbalanced datasets. The attacker is trained and tested on a balanced dataset. Otherwise, the attacker can perform quite well on the male/female task simply by learning to predict sentiment, which does not reflect leakage of gender data to the representation. When training the attacker on balanced data, its decisions cannot rely on the sentiment information encoded in the vectors, and must look for encoded information about the protected attributes.
The results in Table \ref{tbl:leakage} indicate that both task accuracy and attribute leakage are stronger in the unbalanced case.

\paragraph{Leakage: Real-world Example}
The above experiments used artificially constructed datasets. Here, we demonstrate leakage using a popular encoder trained for emotion detection: the DeepMoji encoder \cite{felbo2017} trained to predict the most suitable emoji usage for a sentence (one of 64 in total), based on 1.2 billion tweets. The model is advertised as a good encoder for encoding sentences into a representation that is highly predictive of sentiment, mood, emotion and sarcasm. Does it also capture \textit{protected attributes}? We encode the sentences of the different \textit{protected attributes} using the DeepMoji encoder and train three different \textit{attackers} to predict race, gender and age. The best scores on the development set are 84.7\%, 67.2\% and 67.1\% respectively. This should not come as a surprise, as indeed some emoji usage is highly correlated with these properties.

\section{Mitigating Data Leakage}
Leakage of protected attributes information into the internal representation of the network when training on seemingly unrelated tasks is very common. We explore the means of mitigating such leakage.

\subsection{Adversarial Training}
\label{sec:leakage-adv}
We repeat the experiments in Table \ref{tbl:leakage} with an adversarial component \cite{ganin2015unsupervised} as described in Section \ref{sec:setup}, in order to actively remove the protected attribute information from the encoded representation during training.
Note that the adversarial objective is in odds with the main-task one: by removing the protected attribute information from the encoder, we may also hurt its ability to encode information about the main task.

Figure \ref{fig:adversarial} shows the main task and adversary prediction accuracies on the development set as training progresses, for the Sentiment/Race pair.

\begin{figure}[h!]
   
   \begin{center}
   \includegraphics[scale=0.5]{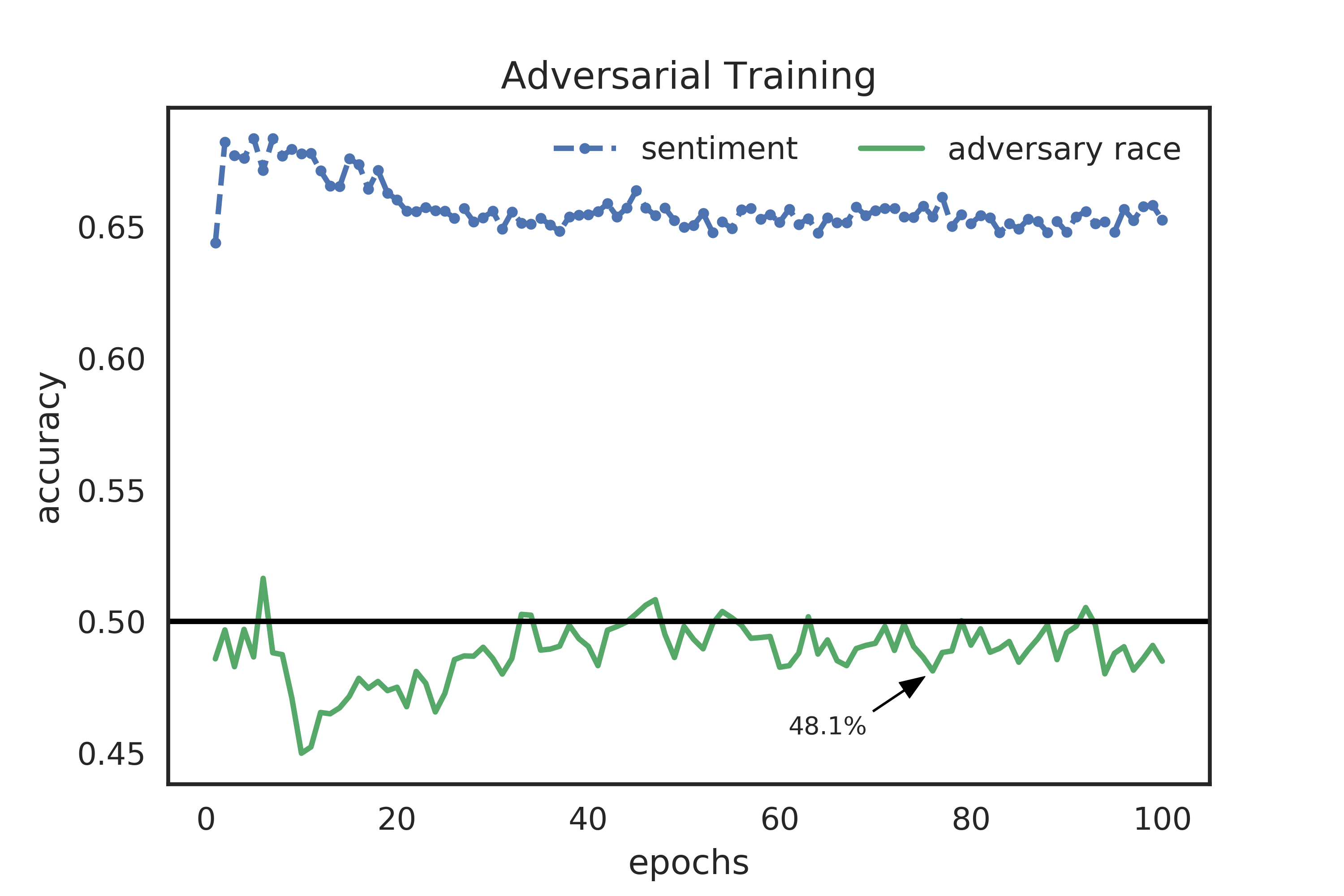}
   \end{center}
   \caption{Main task and Adversary accuracy curves for Sentiment/Race.}
    \label{fig:adversarial}
\end{figure}

After an initial peak in task prediction accuracy, the adversary prediction drops and starts to fluctuate around chance level (50\%), as desired, along with a drop in main task accuracy as well.
The adversary's accuracy remain around chance level throughout the entire training process, suggesting that the adversarial training is indeed effective for removing the protected attribute information from the encoded representation. These trends are persistent for all main-task/protected-attribute pairs we tried.

However, training the attacker network on the resulting encoder vectors reveals a different story. For example, when considering the encoder after 50 training epochs (adversary accuracy of 49.0\%), the attacker reaches 56.0\% accuracy: substantially higher than the adversarial's success rate, despite sharing the exact same architecture, and being trained and tested on the exact same dataset.

Table \ref{tbl:adv} summarizes the attacker's recovery rate on the adversarialy-trained encoders for the different settings. In all cases, the adversarial's success rate is around 50\%, while the attacker's rate is substantially higher. 
The attacker's rates are lower than in Table \ref{tbl:direct-task}, indicating the adversarial training is effective in removing \emph{some} of the protected attribute. However, \textbf{a substantial amount of information managed to leak past the adversary, despite its seemingly perfect performance}.

\input{tables/adv.tex}

\paragraph{Attacker's Accuracy on Unseen Data} 
We verify that the attacker's recovery accuracy persists also on the more realistic scenario in which the attacker is applied to encoded sentences that did not participate in the adversarial training. We constructed an additional dataset of 166K completely unseen samples from the Sentiment/Race case. As expected, the attacker works even better in this case, reaching an accuracy of 59.7\% Vs. 56.0\% on the original development set.

\subsection{Strengthening the Adversarial Component}
\label{sec:improved-adv}

We explore means of strengthening the adversarial component, by tuning its capacity and its weight, as well as by using a novel adversarial-ensemble configuration.

\paragraph{Capacity}
We increase the capacity of the adversarial component by increasing its hidden dimension, while keeping the attacker's hidden dimension constant at 300 dimensions. We try hidden dimensions of size 500, 1000, 2000, 5000 and 8000.

\paragraph{Weight}
We experiment with different weighting of the adversarial component during
training by tuning the $\lambda$ parameter, trying the values 0.5, 1.0 (default), 1.5, 2,
3, 5 (with values above 5 the main task training became extremely unstable, not raising above 50\%).

\paragraph{Ensemble}
An alternative to using larger $\lambda$ values is to introduce several adversaries. The potential benefit of this approach is that rather than focusing harder on removing a single feature, here the different adversaries could each focus on a different aspect of the representation. This approach is potentially better suited to deal with language variability. Concretely, we suggest the following adaptation to the adversarial loss to incorporate $k$ adversaries with different random initializations:
\[
L_{y}(c(h(x)), y) + \sum_{j=1}^{k} L_z(adv_j(g_{\lambda}(h(x))), z)
\]

\input{tables/algo-adv-all.tex}

\paragraph{Other Attempts}
We also experienced with several other techniques:
reinitializing the adversarial weights every $t$ epochs; training the adversary without propagating the error to the encoder components for $t$ epochs and only then starting to propagate; using adversaries with more hidden layers; adding dropout on the encoded vectors and within the encoder. None of these yielded improvements over the above methods.

\paragraph{Results}
All methods are effective to some extent,
Table \ref{tbl:algo-all} summarizes the results.

Increasing the capacity of the adversarial network helped reduce the protected attribute's leakage, though different capacities work best on each setup. On the Sentiment/Race task, none of the higher dimensional adversaries worked better than the 300-dim one, on the PAN16 dataset it did. On PAN16/Gender the 8000-dim adversary performed best, and on PAN16/Age, the 500-dim one. 

Increasing the weight of the adversary through the $\lambda$ parameter also has a positive effect on the result (except on the Sentiment/Race pair).
However, too large $\lambda$ values make training unstable, and require many more epochs for the main-task to stabilize around a satisfying accuracy.

The adversarial ensemble method with 2 adversaries achieves 57.4\% on Sentiment/Race, as opposed to 56.0\% with a single one, but when using 5 different adversaries, we achieve 54.8\%. On the PAN16 dataset larger ensembles are more effective.
However, a potential issue with the ensemble method is that larger ensembles reduces training stability, similar to increasing the $\lambda$ value. For example, with 5 adversaries, the main-task accuracy remained at random for 5 epochs, and only begun rising at the 6th epoch. 
Using 10 adversaries, the main task could not be trained.

To summarize, while all methods are effective to some extent, it appears that (a) no method and parameter setting performs equally well across the different setups; and (b) no method succeeds in completely preventing the leakage of the protected attributes.
Combining the different methods (ensembles of larger networks, larger networks with larger $\lambda$, etc.) did not improve the results.

\input{tables/algo-unbalanced.tex}

\paragraph{Unbalanced Data Results}
We repeated the same set of experiments on the unbalanced Sentiment/Race corpus (Table \ref{tbl:algo-unbalanced}).
In this setup, the results are somewhat similar: increasing the adversarial capacity and $\lambda$ is ineffective, and even increases the attacker's recovery rate. However,
using an ensemble of 5 adversaries does manage to reduce the leakage, but it is still far from a satisfying result.

\section{Analysis}

The gap between the adversary's dev-set accuracy and the after-the-fact attacker accuracy on the same data is surprising. To better understand the phenomenon, we perform further analysis on the Sentiment/Race pair with the default single adversary.

\paragraph{Embedding Vs. RNN}
Recall that the \textit{attacker network} tries to extract as much information from the encoder's output as possible. The encoder consists of two components: (1) Embedding Matrix and (2) an RNN. Therefore, the leakage can be caused due to one of them (or due to their combination).

We conduct the following experiment to determine which part affects  the leakage more: we create a new encoder by composing 2 existing encoders: an encoder with high leakage (\textit{Leaky}, using the baseline encoder) and an encoder with low leakage (\textit{Guarded}, using the 5-Ensemble adversary). We fuse the two encoders by combining the embedding matrix of the \textit{Leaky} encoder with the RNN module of the \textit{Guarded} encoder, and vice versa.
This yields two new encoders: an encoder with a ``leaky'' Embedding Matrix module and a ``strong'' RNN module (\textit{Leaky-EMB}), and an encoder with a ``strong'' Embedding Matrix module and a ``leaky'' RNN module (\textit{Leaky-RNN}). We compare encoders \textit{Leaky-EMB} and \textit{Leaky-RNN} to gauge which module has a greater contribution to the data leakage.
We train attacker-networks over the encoders' output to predict the protected attributes.

\input{tables/merge.tex}

Table \ref{tbl:merge} summarize the results, implying that the leakage is caused mainly by the RNN, and less by the Embedding Matrix.\footnote{A discrepancy exists to some extent in the new encoders, as their parts originate from different models that were trained separately. To test if the fusion is valid, we train a different classifier on top of the new encoders to predict the main task. The combination of the leaked RNN with the guarded embeddings results in 65.4\% on the sentiment task and the other combination results in 60.9\% as opposed to 67.5\% and 63.8\% on the leaked and guarded models, respectively. As the new models are on par with the original ones, we conclude that the new encoders are valid.}

\paragraph{Consistent Leakage: Examples Inspection}
We are interested in tweets whose protected attribute (race) is correctly predicted by the adversary. However, at accuracy rates below 60\%, many of the correct predictions could be attributed to chance. To identify the relevant examples, we repeated the Sentiment/Race default adversary experiment 10 times  with different random seeds. We then trained 10 attacker networks, and used each of them to label all examples in the development set. We then looked for tweets which are consistently and correctly classified by at least 9 attackers.\footnote{776 correct and 946 consistent examples in total} Table \ref{tbl:race-ex} shows some of these cases.  Many of them include tokens (\emph{Naw}, \emph{Bestfrand}, \emph{tan}) and syntactic structures (\emph{Going over Bae house}) which are indeed predictive, though not the most salient features.

\input{tables/race-examples.tex}

\paragraph{Leakage via Embeddings}
Even though we found out the RNN is much more responsible to the leakage then the Embedding, those still contribute to the leakage and are easier to inspect. Therefore, we turn to inspect the encoders' Embedding.
We hypothesize that a possible reason for the adversarial network's inability to completely remove the protected race information is word frequency. Namely, rare words, which might be strongly identified with one group, didn't get enough updates during training and therefore remained predictive towards one of the groups. To quantify this, we compared two vocabularies: words appearing in tweets where the predictions were consistently predicted (9 or 10 out of 10 times) by the different attackers, and words appearing in tweets that were randomly distributed (50\%) between the attackers. If our hypothesis is correct, we expect words from the second group to be more frequent than words in the first group.  We discard words appearing in both groups, and associate each word with its training set frequency. One-tailed Mann-Whitney $U$ test \cite{mann1947test} showed the effect is highly significant with $p<e^{-12}$.

\input{tables/overfitting-ho.tex}

\paragraph{Data Overfitting?}

Standard ML setups often suffer from overfitting on the training data, especially when using neural-networks which tend to memorize the data they encounter. In the adversarial setup, the overfitting could result in the encoder-adversary pair working together to perfectly clean the attributes from the training data, without generalization. Such overfitting could explain the attacker success. Is this what happened?
We test this hypothesis by using the same \textit{attacker networks} experiments solely on the training data.  We train the \textit{attackers} on 90\% of the training data while using the rest 10\% as held-out. If overfitting has occurred, the accuracy is likely to result in 50\% accuracy. Alas, this is not the case. Table \ref{tbl:overfitting} summarize the training accuracies of the \textit{attacker network}. The Mention/Race task achieves the highest score of 64.3\% whereas the Mention/Gender task achieves the lowest - 58.1\%. Even though when trained directly to predict these attributes without the adversarial setup, the training accuracies are much higher, a substantial amount of signal is still left, even in the training data.

\section{Related Work}

The fact that intermediary vector representations that are
trained for one task are predictive of another is not surprising: it is at the core of the success of NLP methods for deriving ``generic'' word and sentence representations (e.g. Word2vec \cite{mikolov2013distributed}, Skip-thought vectors
\cite{kiros2015skip}, Contextualized Word Representations \cite{melamud2016context2vec, peters2018deep} etc.).
While usually considered a positive feature, it can often have undesired consequences one should be aware of and potentially control for.
Several works document biases and stereotypes that are captured by unsupervised word embeddings \cite{bolukbasi2016man,caliskan2017semantics} and ways of mitigating them \cite{bolukbasi2016man,zhang2018mitigating}.
Bias and stereotyping were also documented on a common NLP dataset \cite{rudinger2017social}.
While these work are concerned with the learned representations encoding unwanted biases about the world, our concern is with capturing potentially sensitive demographic information about individual authors of the text.

Removing sensitive attributes (demographic or otherwise) from intermediate representations in order to achieve fair classification has been explored by solving an optimization problem \cite{zemel2013learning}, as well as by employing adversarial training  \cite{edwards2015censoring, louizos2015variational, xie2017controllable, zhang2018mitigating}, focusing on structured features. Adversarial training was also applied for Image anonymization \cite{edwards2015censoring, feutry2018learning}.
In contrast, we consider features that are based on short user-authored text.

Several works apply adversarial training to textual data, in order to learn encoders that are invariant to some properties of the text \cite{chen2016adversarial,conneau2017word, zhang2017aspect,xie2017controllable}. As their main motivation is to remove information about domain or language in order to improve transfer learning, domain adaptation, or end task accuracy, they were less concerned with the ability to recover information from the resulting representation, and did not evaluate it directly as we do here.

Recent work on creating private representation in the text domain \cite{P18-2005} share our motivation of removing unintended demographic attributes from the learned representation using adversarial training. However, they report only the discrimination accuracies of the adversarial component, and do not train another classifier to verify that the representations are indeed clear of the protected attribute. As our work shows, trusting the adversary is insufficient, and external verification is crucial.

Finally, our work is motivated by the desire for fairness. We use a definition in which a fair classification is one that does not condition on a certain attribute (\emph{fairness by blindness}), and evaluate the ability to achieve text-derived representations that are blind to a property we wish to protect. Many other definitions of fairness exist, including \textit{demographic parity}, \textit{equality of odds} and \textit{equality of opportunity} (see e.g. discussion in \cite{hardt2016equality,beutel2017data}). Under our setup, blindness guarantees these metrics (Appendix  \ref{sec:fair-acc}).

\section{Conclusions}
We show that demographic information leaks into intermediate representations of neural networks trained on text data. Systems that train on text data and do not want to condition on demographic information must take active steps against accidental conditioning. Our experiments suggest that:\\
(1) Adversarial training is effective for mitigating protected attribute leakage, but, when dealing with text data, may fail to remove it completely.\\
(2) When using the adversarial training method, the adversary score during training \emph{cannot be trusted}, and must be verified with an externally-trained attacker, preferably on unseen data.\\
(3) Tuning the capacity and weight of the adversary, as well as using an ensemble of several adversaries, can improve the results. However, no single method is the most effective in all cases.

\section*{Acknowledgments}
We would like to thank Moni Shahar, Felix Kreuk, Yova Kementchedjhieva and the BIU NLP lab for fruitful conversation and helpful comments. We also thank Su Lin Blodgett for her help in supplying the \textsc{Dial} dataset and clarifications. This work was supported in part by the The Israeli Science Foundation (grant number 1555/15) and German Research Foundation via the
German-Israeli Project Cooperation (DIP, grant DA 1600/1-1).

\bibliography{emnlp2018}
\bibliographystyle{acl_natbib_nourl}

\clearpage
\appendix

\section{Standard Fairness Definitions and Guarded Classifiers}
\label{sec:fair-acc}

In this work, we focus on creating a representation which is oblivious to some factor. We measure this by the term \textit{leakage} which is defined in Section \ref{sec:setup} and say that a classifier is \textit{guarded} in respect to an attribute $z$ if $z$ is \textit{guarded} and hasn't \textit{leaked}. In this section, we show that this definition matches more common definitions of fairness, under our setup. Specifically, we show that under the setup we discuss, if $z$ is guarded than the classifier satisfies \textit{Demographic Parity}, \textit{Equality of Odds} and \textit{Equality of Opportunity}.\footnote{We note however that, as we discussed, achieving 0-leakage is far from trivial.}

For completeness, we repeat these definitions provided and redefined by Hardt et al. \shortcite{hardt2016equality}, Beutel et al. \shortcite{beutel2017data} and Zhang et al. \shortcite{zhang2018mitigating}:

\paragraph{Demographic Parity.} A predictor $f$ satisfies \textit{demographic parity} if $f$ and $z$ are independent: 
\[P(f=\hat{y}|z=0)=P(f=\hat{y}|z=1)\]

\paragraph{Equality of Odds.} A predictor $f$ satisfies \textit{equality of odds} if $f$ and $z$ are conditionally independent of the main task label $Y$: 

\begin{gather*}
P(f=\hat{y}|z=0,Y=y)=\\
P(f=\hat{y}|z=1,Y=y), y \in \{0,1\}
\end{gather*}

\paragraph{Equality of Opportunity.} A predictor $f$ satisfies \textit{equality of opportunity} if $f$ and $z$ are conditionally independent on a particular value $Y$:
\begin{gather*}
P(f=\hat{y}|z=0,Y=y)=\\
P(f=\hat{y}|z=1,Y=y), y \in \{0|1\}
\end{gather*}

Recall that in all our data split setups, there is an equal appearance of both $y$ and $z$. We now claim the following:

\begin{lemma} \label{lemma:par}
If a classifier is \textit{guarded} to $z$ in our setup, it realizes \textit{demographic parity}
\end{lemma}

\begin{proof}
The classifier is oblivious to $z$, meaning that:
$P(Z|H) = 0.5$, where $H$ is the internal representation. As in our setup, $P(Z)=0.5$
\[\Rightarrow P(Z)=P(Z|H)\]
therefore $Z$ and $H$ are independent.

From graphical models we know that $Y$ and $X$ (the textual input), are independent given $H$ (symmetrically for $Z$ and $X$), therefore we can overlook $X$ when conditioning on $H$.
Also, we can say that given $H$, $\hat{Y}$ and $Z$ are independent:
$$P(\hat{Y},Z|H)=P(\hat{Y}|H)P(Z|H)$$

using bayes rule on both sides:
\begin{gather*}
\frac{P(H|\hat{Y},Z)P(\hat{Y},Z)}{P(H)} = \frac{P(H|\hat{Y})P(\hat{Y})}{P(H)}P(Z|H) \\
\Rightarrow P(H|\hat{Y},Z)P(\hat{Y},Z)=P(H|\hat{Y})P(\hat{Y})P(Z|H)
\end{gather*}

and from the independency of $Z$ and $H$, we get that:
$$P(\hat{Y},Z)=P(\hat{Y})P(Z|H)$$
As $Z$ and $H$ are independent:
$$P(\hat{Y},Z)=P(\hat{Y})P(Z)$$
meaning that $\hat{Y}$ and $Z$ are independent
\end{proof}

\begin{lemma} \label{lemma:odds}
If a classifier is oblivious to $z$ in our setup, it realizes equality of odds.
\end{lemma}

\begin{proof}
If a classifier is \textit{oblivious} to $z$, this means that:
$$P(f=\hat{y}|z=0)=P(f=\hat{y}|z=1)$$ (from Lemma \ref{lemma:par}), therefore, from the law of total probability:
\begin{gather*}
P(f=\hat{y}|z=0,y=0) \cdot P(y=0) + \\
P(f=\hat{y}|z=0,y=1) \cdot P(y=1) = \\
 P(f=\hat{y}|z=1,y=0) \cdot P(y=0) + \\
 P(f=\hat{y}|z=1,y=1) \cdot P(y=1)
\end{gather*}

Since $P(y=0)=P(y=1)=0.5$ in all setups, we can get rid of $P(y)$, and we get:
\begin{gather*}
P(f=\hat{y}|z=0,y=0) + \\ 
P(f=\hat{y}|z=0,y=1) = \\ 
P(f=\hat{y}|z=1,y=0) + \\
P(f=\hat{y}|z=1,y=1)
\end{gather*}
 and due to Lemma \ref{lemma:par} we get
 
\begin{gather*}
P(f=\hat{y}|z=0,y=1) = P(f=\hat{y}|z=1,y=1)
\end{gather*}
(and similarly for $y=0$).
\end{proof}

\begin{lemma}
If a classifier is oblivious to $z$ in our setup, it realizes equality of opportunity.
\end{lemma}

\begin{proof}
As equality of opportunity is a relaxation of \textit{equality of odds} and is less strict then it, from Lemma \ref{lemma:odds}, it holds automatically.
\end{proof}

In conclusion, we showed that an \textit{oblivious classifier} on our setup would satisfy the three fairness definitions which were introduced above.

\section{Implementation Details}
\label{app:imp-det}

\paragraph{Preprocessing}
We tokenize each tweet using \textit{twokenize}, a twitter specific tokenizer \cite{o2010tweetmotif,owoputi2013improved}, and discard duplicate tweets and tweets with less than three tokens.

\paragraph{Neural Network architecture and hyperparameters}
Unless otherwise noted, both the LSTM encoder and the MLP hidden layer have 300 hidden units with a single layer. We use randomly-initialized 300-dimensional embeddings, and train using SGD with Momentum 
\cite{qian1999momentum} and a learning rate of 0.01 for 100 epochs.  We use dropout \cite{hinton2012improving} of 0.2 on all hidden layers, and negative log likelihood as the loss function with 32 sized mini-batch.

\section{Emojis Details}

\label{sec:emojis}
Figure \ref{fig:pos-neg-emojis} contains the different emojis we used for defining positive and negative tweets. In addition to those emojis, we also looked for the following emoticons: :) :-) : ) :D =) as positive and :( :-( : ( :-( =( as negative

\begin{figure}[h!]
 \centering 
 \includegraphics[scale=1.0]{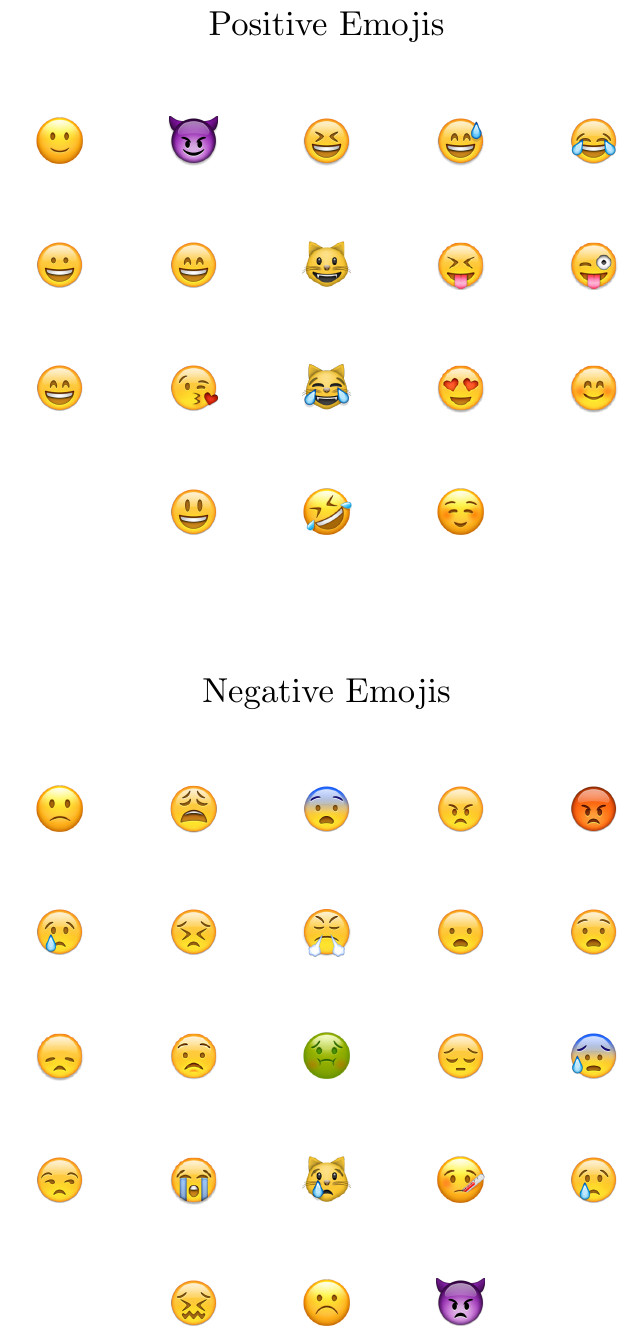}
 \caption{Emojis used as positive and negative proxies for sentiment}
 \label{fig:pos-neg-emojis}
\end{figure}

\end{document}

%% file: tables/upper-bounds.tex
\begin{table}[H]
	\begin{center}
    \footnotesize{
    \begin{tabular}{llc}
        Data & Task & Accuracy \\
        \hline \hline
        \textsc{Dial}& Sentiment		& 67.4 \\
        & Mention   	& 81.2 \\
        & \emph{Race}      	& 83.9 \\
        \hline
        
        PAN16 & Mention	  	& 77.5 \\
        & \emph{Gender} 	 	& 67.7 \\
        & \emph{Age}		  	& 64.8
    \end{tabular}

    }
    \end{center}
    \caption{Accuracies when training directly towards a single task.}
    \label{tbl:direct-task}
\end{table}

%% file: tables/leakage.tex
\begin{table*}[t!]
\begin{center}
\footnotesize{

\begin{tabular}{lll|cc|cc}
      &           &                     & \multicolumn{2}{c}{Balanced} & \multicolumn{2}{c}{Unbalanced} \\
Data  & Task      & Protected Attribute & Task Acc      & Leakage      & Task Acc       & Leakage       \\ \hline \hline
\textsc{Dial}  & Sentiment & Race                & 67.4          & 64.5         & 79.5           & 73.5          \\
      & Mention   & Race                & 81.2          & 71.5         & 86.0           & 73.8          \\ \hline
PAN16 & Mention   & Gender              & 77.5          & 60.1         & 76.8           & 64.0          \\
      &           & Age                 & 74.7          & 59.4         & 77.5           &    59.7      
\end{tabular}

}
\end{center}
\caption{Protected attribute leakage: balanced \& unbalanced data splits.}
\label{tbl:leakage}
\end{table*}

%% file: tables/adv.tex
\begin{table}[ht!]
\begin{center}

\footnotesize{

\begin{tabular}{lllrrr}
Data  	& Task & \begin{tabular}[c]{@{}l@{}}Protected\\ Attribute\end{tabular} & \begin{tabular}[c]{@{}l@{}}Task\\ Acc\end{tabular}  & Leakage & $\Delta$ \\ \hline \hline
\textsc{Dial} 	& Sentiment & Race      & 64.7      & 56.0 & 5.0   \\
	 	& Mention   & Race      & 81.5      & 63.1 & 9.2   \\ \hline
PAN16 	& Mention	& Gender	& 75.6 		& 58.5 & 8.0 	 \\
	  	& Mention   & Age		& 72.5		& 57.3 & 6.9
\end{tabular}
}
\end{center}
\caption{Performances on different datasets with an adversarial training. $\Delta$ is the difference between the attacker score and the corresponding adversary's accuracy.}
\label{tbl:adv} 
\end{table}

%% file: tables/algo-adv-all.tex
\begin{table*}[t]
\begin{center}
\footnotesize{
\begin{tabular}{lr|ccc|ccc|ccc}
                      & \multicolumn{1}{l|}{} & \multicolumn{3}{c|}{\textsc{Dial}}                    & \multicolumn{6}{c}{PAN16}                                                        \\
Method             & Parameter                 & Sentiment & Race            & $\Delta$       & Mention & Gender          & $\Delta$       & Mention & Age             & $\Delta$       \\ \hline \hline
No Adversary Baseline & -     & 67.4   & 14.5  & -     & 77.5 & 10.1 & -   	& 74.7 & 9.4           & -              \\
Standard Adversary    & (300/1.0/1)   & 64.7   & 6.0  & 5.0 & 75.6 & 8.5 & 8.0 & 72.5 & 7.3 & 6.9  \\ \hline
Adv-Capacity          
&  500 		& 64.1 & 6.7 & 5.2 & 73.8 & 8.1 & 6.7 & 71.4 & \textbf{4.3} & 4.1 \\
&  1000 	& 63.4 & 7.1 & 4.9 & 75.2 & 8.9 & 7.0 & 71.6 & 6.3 & 4.0 \\
&  2000 	& 65.2 & 8.1 & 6.9 & 76.1 & 6.7 & 6.4 & 71.9 & 6.0 & 5.7 \\
&  5000 	& 63.9 & \textbf{6.2} & 3.7 & 74.5 & 5.6 & 1.6 & 73.0 & 10.2 & 9.6 \\
&  8000 	& 65.0 & 7.1 & 4.8 & 75.7 & \textbf{5.4} & 4.2 & 71.9 & 9.8 & 7.3 \\
 \hline
 
$\lambda$                
&  0.5 	& 63.9 & \textbf{6.8} & 6.2 & 75.6 & 7.8 & 6.8 & 73.1 & 4.8 & 3.4 \\
&  1.5 	& 64.9 & 7.4 & 5.4 & 75.6 & \textbf{4.9} & 2.4 & 72.5 & 6.8 & 5.8 \\
&  2.0 	& 64.2 & 7.3 & 5.9 & 76.0 & -7.2 & 6.7 & 72.1 & 8.5 & 7.7 \\
&  3.0 	& 65.8 & 10.2 & 10.1 & 73.7 & 6.4 & 6.1 & 72.5 & -6.3 & 5.2 \\
&  5.0 	& 50.0 & - & - & 73.6 & 6.5 & 5.7 & 69.0 & \textbf{3.2} & 2.9 \\ \hline
Ensemble              
&  2 	& 62.4 & 7.4 & 5.4 & 74.8 & 6.4 & 5.0 & 72.8 & 8.8 & 8.3 \\
&  3 	& 66.5 & 6.5 & 5.0 & 75.3 & 4.9 & 3.1 & 72.1 & 6.7 & 6.0 \\
&  5 	& 63.8 & \textbf{4.8} & 2.6 & 74.3 & \textbf{4.1} & 3.0 & 70.1 & \textbf{5.7} & 5.4 \\        
\end{tabular}
}
\end{center}
\caption{\label{tbl:algo-all} Results of different adversarial configurations.
\textbf{Sentiment/Mention}: main task accuracy.  \textbf{Race/Gender/Age}: protected attribute recovery difference from 50\% rate by the attacker (values below 50\% are as informative as those above it). $\mathbf{\Delta}$: the difference between the attacker score and the corresponding adversary's accuracy. The bold numbers are the best \textit{oblivious} classifiers within each configuration.}
\end{table*}

%% file: tables/algo-unbalanced.tex
\begin{table}[t]
\begin{center}
\footnotesize{
\begin{tabular}{lr|rcc}
Method        & Param & Sentiment & Race \\ \hline \hline
No Adversary Baseline       & - & 79.5        & 23.5 \\
Standard Adversary			&  1.0 	& 76.8 &  10.6 \\
\hline
Adv-Capacity	
&  500 	& 74.8 & \textbf{13.8} \\
&  1000 	& 70.5 & 18.4 \\
&  2000 	& 73.9 & 18.5 \\
&  5000 	& 71.5 & 19.4 \\
&  8000 	& 73.6 & 18.7 \\
\hline
Lambda 			
&  0.5 	& 75.0 & 15.5 \\
&  1.5 	& 71.2 & 18.2 \\
&  2.0 	& 73.0 & 12.1 \\
&  3.0 	& 71.5 & \textbf{12.0} \\
&  5.0 	& 50.0 & - \\
\hline
Ensemble   		
&  2 	& 70.6 & 20.8 \\
&  3 	& 73.6 & 17.9 \\
&  5 	& 71.5 & \textbf{8.6} \\
\end{tabular}
}
\end{center}
\caption{Unbalanced Sentiment/Race with the different methods. \textbf{Sentiment}: task accuracy. \textbf{Race}: Attacker's recovery accuracy beyond 50\%.}
\label{tbl:algo-unbalanced}
\end{table}

%% file: tables/merge.tex
\begin{table}[h]
\centering
\begin{tabular}{llcc}
                     &         & \multicolumn{2}{c}{Embedding} \\ \cmidrule{3-4}
                     &         & Leaky        & Guarded        \\
\multicolumn{1}{l|}{\multirow{2}{*}{\begin{turn}{90} RNN \end{turn}}} & Leaky  & 64.5              & 67.8               \\ 
\multicolumn{1}{l|}{}     & Guarded & 59.3              & 54.8              
\end{tabular}
\caption{Accuracies of the protected attribute with different encoders.}
\label{tbl:merge}
\end{table}

%% file: tables/race-examples.tex
\begin{table*}[t]
\begin{center}
\footnotesize{
\begin{tabular}{l|l}
AAE (``non-hispanic blacks'')                                              & SAE  (``non-hispanic whites'')                                            \\ \hline
My Brew Eattin                          & I want to be tan again \\
\_ Naw im cool                          & Why is it so hot in the house ?! \\
Tonoght was cool                        & Been doing Spanish homework for 2 hours .              \\
My momma Bestfrand died                 & I wish I was still in Spain \\
Enoy yall day                           & Ahhhhh so much homework . \\
Going over Bae house                    & \_TWITTER-ENTITY\_ I miss you too ! \\
She not texting or calling ? Ok         & I want to move to california \\
Real relationships go thru real shit 	& Lol , I don't even go here .           \\
About to spend my entire check IDGAF    & Ahhhhh so much homework .\\
Getting ready for school                & I'm so tired .
\end{tabular}
}
\end{center}
\caption{Examples for correct dialectal/race predictions, which were predicted consistently by at least 9 different attacker-classifiers.}
\label{tbl:race-ex}
\end{table*}

%% file: tables/overfitting-ho.tex
\begin{table}[ht!]
\begin{center}

\footnotesize{

\begin{tabular}{lllrrr}
Data  	& Task & \begin{tabular}[c]{@{}l@{}}Protected\\ Attribute\end{tabular} &  $\Delta$ \\ \hline \hline
\textsc{Dial} 	& Sentiment & Race      & 12.2   \\
	 	& Mention   & Race      & 14.3   \\ \hline
PAN16 	& Mention	& Gender	& 8.1 	 \\
	  	& Mention   & Age		& 9.7
\end{tabular}
}
\end{center}
\caption{Attacker's performance on different datasets. Results are on a training set 10\% held-out. $\Delta$ is the difference between the attacker score and the corresponding adversary's accuracy.}
\label{tbl:overfitting} 

\vspace{-0.2cm}

\end{table}